\documentclass{article}

\usepackage{microtype,enumitem}
\usepackage{graphicx}
\usepackage{subfigure}
\usepackage{booktabs} %

\usepackage{amsmath,amsfonts,bm}

\def\eqref#1{Equation~\ref{#1}}

\def\1{\bm{1}}

\def\mvarepsilon{{\boldsymbol{\varepsilon}}}

\def\vu{{\bm{u}}}

\def\vw{{\bm{w}}}
\def\vx{{\bm{x}}}
\def\vy{{\bm{y}}}

\def\mA{{\bm{A}}}

\def\mH{{\bm{H}}}
\def\mI{{\bm{I}}}

\def\mS{{\bm{S}}}

\def\mU{{\bm{U}}}

\def\mX{{\bm{X}}}

\def\mSigma{{\bm{\Sigma}}}

\DeclareMathAlphabet{\mathsfit}{\encodingdefault}{\sfdefault}{m}{sl}
\SetMathAlphabet{\mathsfit}{bold}{\encodingdefault}{\sfdefault}{bx}{n}

\def\gD{{\mathcal{D}}}

\def\gF{{\mathcal{F}}}
\def\gG{{\mathcal{G}}}

\def\gL{{\mathcal{L}}}

\def\gN{{\mathcal{N}}}

\def\sR{{\mathbb{R}}}

\newcommand{\E}{\mathbb{E}}

\DeclareMathOperator*{\argmin}{arg\,min}

\usepackage[linesnumbered,ruled,vlined]{algorithm2e}
\usepackage{amssymb,amsmath}%
\usepackage{pifont}%
\usepackage{tikz}

\usepackage{amsmath}
\usepackage{amsthm}

\newtheorem{assumption}{Assumption}[section]

\newtheorem{theorem}{Theorem}[section]
\newtheorem{corollary}[theorem]{Corollary}
\newtheorem{lemma}[theorem]{Lemma}

\newcommand{\tr}{\mathrm{tr}}

\newcommand{\Rnum}[1]{\uppercase\expandafter{\romannumeral #1\relax}}

\def\rone{\expandafter{\romannumeral1}}
\def\rtwo{\expandafter{\romannumeral2}}
\def\rthree{\expandafter{\romannumeral3}}

\definecolor{citecolor}{HTML}{0071BC}
\definecolor{linkcolor}{HTML}{ED1C24}
\definecolor{pku}{RGB}{139 0 18}

\newcommand{\wt}{\widetilde}
\newcommand{\wh}{\widehat}

\usepackage{hyperref}
\usepackage[normalem]{ulem}

\usepackage[accepted]{icml2024}

\usepackage{amsmath}
\usepackage{amssymb}
\usepackage{mathtools}
\usepackage{amsthm}

\renewcommand{\eqref}[1]{(\ref{#1})}

\theoremstyle{plain}

\usepackage[textsize=tiny]{todonotes}
\graphicspath{icml2024}

\icmltitlerunning{A Statistical Theory of Regularization-Based Continual Learning}

\begin{document}

\twocolumn[
\icmltitle{A Statistical Theory of Regularization-Based Continual Learning}

\icmlsetsymbol{equal}{*}

\begin{icmlauthorlist}
\icmlauthor{Xuyang Zhao}{pku}
\icmlauthor{Huiyuan Wang}{upenn}
\icmlauthor{Weiran Huang \textsuperscript{\dag}}{sjtu,ailab}
\icmlauthor{Wei Lin \textsuperscript{\dag}}{pku}
\end{icmlauthorlist}

\icmlaffiliation{pku}{School of Mathematical Sciences and Center for Statistical Science, Peking University, Beijing, China}
\icmlaffiliation{upenn}{Department of Biostatistics, Epidemiology and Informatics, Perelman School of Medicine, University of Pennsylvania, Philadelphia, PA, USA}
\icmlaffiliation{sjtu}{MIFA Lab, Qing Yuan Research Institute, SEIEE, Shanghai Jiao Tong University, Shanghai, China}
\icmlaffiliation{ailab}{Shanghai AI Laboratory, Shanghai, China}

\icmlcorrespondingauthor{Weiran Huang}{weiran.huang@outlook.com}
\icmlcorrespondingauthor{Wei Lin}{weilin@math.pku.edu.cn}

\icmlkeywords{Catastrophic forgetting, continual learning, heterogeneity, regularization}

\vskip 0.3in
]

\printAffiliationsAndNotice{}  %
 
\begin{abstract}
We provide a statistical analysis of regularization-based continual learning on a sequence of linear regression tasks, with emphasis on how different regularization terms affect the model performance. We first derive the convergence rate for the oracle estimator obtained as if all data were available simultaneously. Next, we consider a family of generalized $\ell_2$-regularization algorithms indexed by matrix-valued hyperparameters, which includes the minimum norm estimator and continual ridge regression as special cases. As more tasks are introduced, we derive an iterative update formula for the estimation error of generalized $\ell_2$-regularized estimators, from which we determine the hyperparameters resulting in the optimal algorithm. Interestingly, the choice of hyperparameters can effectively balance the trade-off between forward and backward knowledge transfer and adjust for data heterogeneity. Moreover, the estimation error of the optimal algorithm is derived explicitly, which is of the same order as that of the oracle estimator. In contrast, our lower bounds for the minimum norm estimator and continual ridge regression show their suboptimality. A byproduct of our theoretical analysis is the equivalence between early stopping and generalized $\ell_2$-regularization in continual learning, which may be of independent interest. Finally, we conduct experiments to complement our theory.
\end{abstract}

\section{Introduction}
Continual learning (CL) in machine learning involves training a model continuously across multiple tasks, constrained by limited memory. As more tasks are introduced and additional data samples are collected, it is expected that the model will exhibit enhanced performance on both old and new tasks. 
However, due to memory limits, not all past data can be retained; typically, only a subset of the data or summary statistics are stored. This makes continual learning more challenging than single-task learning, as it prohibits the simple pooling of all samples \citep{parisi2019continual}.
Alternatively, without using exceedingly large long-term memory, we can view continual learning as an online multi-task problem where a model is sequentially fitted to data provided for each task. However, such an approach may result in poor performance of the current model on previous tasks, a phenomenon known as \textit{catastrophic forgetting} \citep{mccloskey1989catastrophic, goodfellow2013empirical}. Clearly, forgetting information from earlier tasks undermines the overall effectiveness of the model.

There are two goals of continual learning algorithms. One is the \textit{forward knowledge transfer}, which focuses on transferring knowledge from previous tasks to make learning on new tasks simpler.
The other is the \textit{backward knowledge transfer} \cite{lin2023theory}, which aims to address the issue of catastrophic forgetting when learning new tasks and keep the overall performance improving over time.
From a statistical perspective, the main difficulty in these two goals is \textit{heterogeneity} among tasks, i.e., the data distribution can vary across different tasks.
In the presence of heterogeneity, the forward and backward knowledge transfer can contradict each other, between which a trade-off will arise \cite{lin2023theory,wang2023comprehensive}. An ideal CL algorithm should properly balance the knowledge extracted from old tasks and the information contained in new samples to achieve both forward and backward knowledge transfer.

To resolve the conflict, many algorithms have been proposed recently. Roughly speaking, these algorithms fall into three categories: regularization-based methods \citep{kirkpatrick2017overcoming, aljundi2018memory, liu2022continual}, replay-based methods \citep{chaudhry2018efficient,riemer2018learning,jin2021gradient}, and expansion-based methods \citep{serra2018overcoming, yoon2019scalable,yang2021grown}.
The common intuition underlying these algorithms is applying different techniques that can use old information to constrain the model's change on new tasks, thereby achieving forward and backward knowledge transfer simultaneously.
However, the theoretical understanding of CL algorithms is still underdeveloped. In particular, none of the existing work shows an explicit trade-off between forward and backward knowledge transfer, let alone offering a guidance on how to balance them properly. Also, the roles of heterogeneity and noise are not fully discussed, which are crucial aspects of practical continual learning.

In this paper, we enrich the existing literature by establishing theoretical properties of regularization-based continual learning algorithms within the linear regression framework. Our analysis includes considerations for heterogeneity, noise, and overparametrization, and offers an in-depth investigation of the trade-off between forward and backward knowledge transfer.
Specifically, our contributions are summarized as follows.
\begin{itemize}
\item We provide lower bounds for two continual learning algorithms, i.e., the minimum norm estimator \citep{lin2023theory} and continual ridge regression \citep{li2023fixed}.
These bounds reveal their suboptimality compared to the oracle estimator, which motivates us to study some new algorithms.

\item We point out two main reasons for the failure of the above two methods: forward--backward trade-off and information heterogeneity.
The former is essentially the trade-off between the information carried in old tasks and that in the new task, and the latter means that the knowledge carried in different tasks varies.

\item  Inspired by our findings, we propose a generalized $\ell_2$-regularized estimator.
By choosing its hyperparameters properly to deal with the forward--backward trade-off and information heterogeneity, we show that our estimator attains the error rate of the oracle estimator and hence avoids catastrophic forgetting.

\item We establish the relationship between early stopping between $\ell_2$-regularization in continual linear regression.
We show that, if the learning rate of gradient descent takes a more general form as in our generalized $\ell_2$-regularization, then these two methods are actually equivalent.
This can be viewed as an extension of similar results shown for learning a single task.

\item We conduct simulation experiments to complement our theory. We obtain a practical algorithm based on the above theoretical results, which has a close connection with elastic weighted consolidation (EWC). We illustrate its performance through simulations.
\end{itemize}
\subsection{Related Work}
\paragraph{Continual learning algorithms.} Over the past several years, continual learning has attracted considerable attention, leading to the proposal of numerous empirical algorithms aimed at mitigating catastrophic forgetting. Broadly speaking, these methods can be categorized into three groups: (1) regularization-based methods \citep{kirkpatrick2017overcoming, aljundi2018memory, liu2022continual}, which regularize modifications to the importance weights for old tasks when learning the new task; (2) expansion-based methods \citep{serra2018overcoming, yoon2019scalable,yang2021grown}, which learn a mask to fix the importance weights for old tasks during the new task learning and further expand the neural network when needed; (3) memory-based methods, which either store and replay the data from old tasks when learning the new task, i.e., experience-replay based methods \citep{chaudhry2018efficient,riemer2018learning,jin2021gradient}, or store the gradient information from old tasks and learn the new task in the direction orthogonal to old tasks, i.e., orthogonal-projection based methods \citep{farajtabar2020orthogonal,saha2021gradient,lin2022trgp}.

\paragraph{Theoretical studies in CL.} 
\citet{mccloskey1989catastrophic} proposed a unified framework for the performance analysis of regularization-based CL methods, by formulating them as a second-order Taylor approximation of the loss function for each task.
\citet{bennani2020generalisation} and \citet{doan2021theoretical} analyzed generalization error and forgetting for the orthogonal gradient descent (OGD) approach \cite{yin2020optimization} based on NTK models, and further proposed variants of OGD to address forgetting. \citet{lee2021continual} and \citet{asanuma2021statistical} studied CL in the teacher--student setup to characterize the impact of task similarity on forgetting performance. \citet{cao2022provable} and \citet{li2022provable} investigated continual representation learning with dynamically expanding feature spaces, and developed provably efficient CL methods with a characterization of the sample complexity.

Besides, there are some theoretical works on regularization-based methods.
\citet{evron2022catastrophic} studied the minimum norm estimator in CL under an overparameterized and noise-free setup. \citet{li2023fixed} gave a fixed design analysis of continual ridge regression for two-task linear regression.

\citet{chen2022memory} characterized the lower memory bound in CL using the PAC framework. \citet{andle2022theoretical} analyzed the selection of frozen filters based on layer sensitivity to maximize the performance of CL. \citet{wen2024provable} studied the contrastive CL methods and provided upper and lower performance bounds.
\citet{yang2022optimizing} presented a CL algorithm based on supervised PCA and gave a theoretical analysis.
\citet{denevi2019learning} proposed to add a bias term to SGD and showed improved performance theoretically.

\section{Continual Linear Regression}\label{sec:setup}

\paragraph{Data.}
We consider a standard continual learning problem where a sequence of tasks indexed by $t = 1,\dots, T$ arrives sequentially.
Suppose that each task $t$ holds a dataset $\gD_t = \{(\vx_i^{(t)}, y_i^{(t)})\in\sR^{p}\times \sR\}_{i=1}^{n_t}$, where $n_t$ denotes its sample size. 
We assume that all of the $T$ tasks are generated by a linear model with the same regression coefficient, i.e., for all $t\in[T]$ and $i\in[n_t]$, 
\begin{equation}\label{eq:data}
y_i^{(t)} = (\vx_i^{(t)})^\top \vw_* + \varepsilon_i^{(t)},
\end{equation}
where $\vw_*\in\sR^p$ is the true parameter and $\varepsilon_i^{(t)}$ are independent random noises with variance $\sigma^2$. 
By stacking the data as $\mX_t:=(\vx_1^{(t)}, \dots,\vx_{n_t}^{(t)})^\top\in\sR^{n_t\times p}$, $\vy_t:=(y_{1}^{(t)},\dots,y_{n_t}^{(t)})\in\sR^{n_t}$, and $\boldsymbol{\varepsilon}_t:=(\varepsilon_{1}^{(t)},\dots,\varepsilon_{n_t}^{(t)})\in\sR^{n_t}$, we can rewrite \eqref{eq:data} as
$$
\vy_t = \mX_t \vw_* + \boldsymbol{\varepsilon}_t.
$$
We define $\boldsymbol{\Sigma}_t := \mX_t^\top \mX / n_t \in \sR^{p\times p}$ as the covariance matrix for task $t$. 
Note that we do not require $n_t>p$, i.e., we allow for overparametrization in any single task.

\paragraph{Evaluation metric.}
Our goal is to estimate $\vw_*$ in the continual learning setting.
For any estimator $\wh{\vw}$, we use $\gL(\wh{\vw}):= E \|\wh{\vw}-\vw_*\|^2$ to denote its estimation error.
Note that the definition of $\gL$ applies to each task, since they share a common true parameter $\vw_*$. 
Based on $\gL$, two key metrics, forgetting and generalization error, can be defined respectively as
$$
\begin{aligned}
&\gF_t := \frac{1}{t-1}\sum_{\tau=1}^{t-1} (\gL(\wh{\vw}_t) - \gL(\wh{\vw}_\tau)),\\
&\gG_t := \frac{1}{t}\sum_{\tau=1}^t \gL(\wh{\vw}_t) = \gL(\wh{\vw}_t ),
\end{aligned}
$$
for each $t\in[T]$, 
where $\wh{\vw}_\tau$ denotes the output of a continual learning algorithm after the arrival of task $\tau$.
Small $\gF_t$ means that the estimator learned after task $t$ still has good performance on previous tasks. If $\gF_t < 0$ for every $t\in[T]$, the continual learning algorithm achieves consistently increasing performance and avoids catastrophic forgetting.

\paragraph{Oracle estimator.}
Without the constraint of continual learning, i.e., data of all tasks are available simultaneously, we can estimate $\vw_*$ by simply pooling all samples together and solving the offline optimization problem
$$
\min_{\vw}\left\{ \sum_{t=1}^T \|\mX_t \vw - \vy_t\|^2\right\}.
$$
We call its solution the oracle estimator (ORA) and denote it by 
\begin{equation}\label{eq:ora}
\wh{\vw}_T^{\mathrm{(ORA)}} := \argmin_{\vw} \left\{\sum_{t=1}^T \|\mX_t \vw - \vy_t\|^2\right\}.
\end{equation}
The oracle estimator cannot be used in continual learning practice since it requires simultaneous availability of all data.
Nevertheless, it serves as an ideal baseline to gauge the accuracy of estimating $\vw_*$ without continual learning constraint.
If a continual learning algorithm exhibits comparable performance to the oracle estimator, then we can assert the superiority of that algorithm.

\section{Learning Algorithms}\label{sec:lower}
In this paper, our primary objective is to investigate the \textit{generalized $\ell_2$-regularization} algorithm (GR), which is a family of regularization-based continual learning algorithms. Specifically, it sequentially produces an estimate of $w_*$ as depicted in Algorithm \ref{alg:example}, where $\{\mH_t\}_{t=1}^T$ are user-specified regularization weight matrices and $\|\vw-\wh{\vw}_{t-1}^{(\mathrm{GR})}\|_{\mH_{t}}^2 := (\vw-\wh{\vw}_{t-1}^{(\mathrm{GR})})^\top {\mH_{t}} (\vw-\wh{\vw}_{t-1}^{(\mathrm{GR})})$.

\begin{algorithm}[ht]
\caption{Generalized $\ell_2$-regularization method}
\label{alg:example}
\begin{algorithmic}
\STATE {\bfseries Initialization:} $\wh{\vw}_0^{(\mathrm{GR})} = 0$
\STATE {\bfseries Iterative update for each task $t\in[T]$:}
\begin{equation}\label{eq:gr}
\begin{aligned}
\wh{\vw}_t^{(\mathrm{GR})} :=& \argmin_{\vw} \biggl\{\frac{1}{n}\|\mX_t \vw - \vy_t\|^2\\
&\phantom{\argmin_{\vw} \biggl\{}\quad{}+ \|\vw-\wh{\vw}_{t-1}^{(\mathrm{GR})}\|_{\mH_{t}}^2\biggr\}
\end{aligned}
\end{equation}
\end{algorithmic}
\end{algorithm}

The choice of $\{\mH_t\}_{t=1}^T$ determines how we navigate the balance between forward and backward knowledge transfer. With different choices of $\{\mH_t\}_{t=1}^T$, the GR algorithm encompasses several commonly studied algorithms as special cases. For example, when $\mH_t=\lambda_t \mI_p$ for some $\lambda_t>0$, GR becomes the conventional continual ridge regression algorithm \cite{li2023fixed}.
In the overparameterized scenario where $p>n_t$, if $\mH_t\to 0$, then GR is equivalent to the minimum norm estimator \cite{evron2022catastrophic,lin2023theory}.

In the rest of this section, we give an in-depth discussion of these two algorithms.

\subsection{Minimum Norm Estimator}
Let $\gamma_j^{(t)}$ be the $j$th eigenvalue of $\boldsymbol{\Sigma}_t$.
If $|\{j:\gamma_j^{(t)} > 0\}| = n_t < p$, then there always exists some $\vw$ that interpolates the training data of task $t$, i.e., $\boldsymbol{X}_t \vw = \boldsymbol{y}_t$. 
In this overparameterized regime, some recent works \cite{evron2022catastrophic,lin2023theory} studied the the minimum norm estimator (MN), which is defined in Algorithm \ref{alg:mn}.

\begin{algorithm}[htbp]
\caption{Minimum norm estimator}
\label{alg:mn}
\begin{algorithmic}
\STATE {\bfseries Initialization:} $\wh{\vw}_0^{\mathrm{(MN)}} = 0$
\STATE {\bfseries Iterative update for each task $t\in[T]$:}
$$
\begin{aligned}
\wh{\vw}_t^{\mathrm{(MN)}} = \argmin_{\vw} &\biggl\{\|\vw-\wh{\vw}_{t-1}^{\mathrm{(MN)}}\|^2\\
&\phantom{\biggl\{}\quad\text{ s.t. }\boldsymbol{X}_t \vw = \boldsymbol{y}_t,\biggr\}
\end{aligned}
$$
\end{algorithmic}
\end{algorithm}

Compared to $\ell_2$-regularization methods, MN can be regarded as the limit of the $\ell_2$-regularized estimator when the penalty strength tends to $0$.
From this perspective, it might overly prioritize the data from the new task and underestimate the knowledge embedded in old tasks.
Given that $\boldsymbol{y}_t = \boldsymbol{X}_t \vw_* + \boldsymbol{\varepsilon}_t \not=\boldsymbol{X}_t \vw_*$, imposing the condition $\boldsymbol{X}_t \vw = \boldsymbol{y}_t$ on the estimators inevitably introduces the noise term, which in reality dominates the information when $p>n$.

Specifically, the following theorem provides a lower bound showing that the estimation error of the MN estimator cannot converge to $0$.

\begin{theorem}[Lower bound for the minimum norm estimator]\label{thm:min-norm}
Suppose that $\boldsymbol{\Sigma}_t$ satisfies $|\{j:\gamma_j^{(t)}>0\}| = n_t < p$. Then we have
$$
\gL(\wh{\vw}_t^{\mathrm{(MN)}}) \geq \frac{\sigma^2}{\max_{j\in[p]} \gamma_j^{(t)}}.
$$
\end{theorem}

From Theorem \ref{thm:min-norm} we see that the estimation error of the minimum norm estimator is lower bounded by a term independent of the old tasks.

Consequently, even if the old tasks provide sufficient samples for an accurate estimate of $\vw_*$ or the number of tasks increases infinitely, 
the estimation error of the MN estimator is always lower bounded by a constant that is not approaching $0$.
Indeed, irrespective of the accuracy of $\wh{\vw}_{t-1}$, even if it precisely matches $\vw_*$, the MN estimator cannot leverage it to obtain a better estimate. This is because the estimator attempts to interpolate the newly encountered deficient data and hence does not balance the trade-off between old and new tasks, which we refer to as the \emph{forward--backward trade-off}. 
As a result, the MN estimator is highly susceptible to catastrophic forgetting.

\subsection{Continual Ridge Regression}

Continual ridge regression (CRR) \cite{li2023fixed} uses ridge regularization to constrain the parameter's change when fitting new tasks. Specifically, it updates the estimate using the iterations defined in Algorithm \ref{alg:CRR}.

\begin{algorithm}[htbp]
\caption{Continual ridge regression}
\label{alg:CRR}
\begin{algorithmic}
\STATE {\bfseries Initialization:} $\wh{\vw}_0^{\text{(CRR)}} = 0$
\STATE {\bfseries Iterative update for each task $t\in[T]$:}
$$
\begin{aligned} \wh{\vw}_t^{\text{(CRR)}} = \argmin_{\vw}\biggl\{&\frac{1}{n}\|\boldsymbol{X}_t \vw - \boldsymbol{y}_t\|^2\\
&\quad{}+ \lambda_t \|\vw - \wh{\vw}_{t-1}\|^2\biggr\}
\end{aligned}
$$
\end{algorithmic}
\end{algorithm}

Clearly, CRR is a special case of our generalized $\ell_2$-regularized estimator, which uses the conventional ridge penalty by setting $\lambda_t^{(1)} = \dots = \lambda_t^{(p)} = \lambda_t$.
CRR treats each coordinate of $\vw_*$ equally, i.e., it potentially assumes that $|(\wh{\vw}_t^{(\mathrm{CRR})})_j-(\vw_*)_j|^2$ are the same for different $j$.

However, such homogeneity does not always exist in continual learning setting since the information introduced by different tasks can vary across various directions of $\vw_*$, especially in the scenario where the data distributions differ across tasks.
For example, there may exist some $i\not=j$ such that $|(\wh{\vw}_t^{(\mathrm{CRR})})_j-(\vw_*)_j|^2 = o(1)$ while $|(\wh{\vw}_t^{(\mathrm{CRR})})_i-(\vw_*)_i|^2 = O(1)$ if previous tasks contain very little information about $(\vw_*)_i$. In this case, the suitable values for $\lambda_i$ and $\lambda_j$ might differ. Consequently, the CRR estimator, which cannot address such \textit{information heterogeneity}, may be suboptimal.

More specifically, we have the following lower bound for the CRR estimator, which shows that its worst-case performance is much worse than that of GR.

\begin{theorem}[Lower bound for continual ridge regression]\label{thm:ridge}
Consider a two-task and two-dimensional
continual learning problem with covariance matrices $\boldsymbol{\Sigma}_1 = \mathrm{diag}(1, \epsilon)$ and $\boldsymbol{\Sigma}_2 = \mathrm{diag}(\epsilon, 1)$ and sample sizes $n_1$ and $n_2$. Then we have
$$
\sup _{n_1, n_2, \epsilon} \inf _\lambda \frac{\mathcal{L}\left(\widehat{\boldsymbol{w}}_2^{(\mathrm{CRR})}\right)}{\mathcal{L}\left(\widehat{\boldsymbol{w}}_2^{(\mathrm{GR})}\right)}= +\infty,
$$
where $\lambda$ is the regularization hyperparameter of CRR.
\end{theorem}

\section{Generalized $\ell_2$-Regularization Attains Oracle Rate}\label{sec:main}
In this section, we provide a theoretical analysis of the generalized $\ell_2$-regularized estimator (GR) defined in \eqref{eq:gr}. 
Our theory shows that, through the proper selection of the regularization weight matrix $\mH_t$, it is possible to avoid catastrophic forgetting, and the resulting estimation error can even be comparable with that of the oracle estimator.

Before establishing our main results, we first present some assumptions for our analysis.

\subsection{Assumptions}
\begin{assumption}[Fixed design]\label{assump:fixed_design}
The features $\{\mX_t\}_{t=1}^T$ are fixed while the noises $\boldsymbol{\varepsilon}_t$ are random with mean $0$ and variance $\sigma^2>0$.
\end{assumption}

\begin{assumption}[Commutable covariance matrices]\label{assump:commutable}
The set of covariance matrices $\{\boldsymbol{\Sigma}_t\}_{t=1}^T$ are commutable.
\end{assumption}
These two assumptions ensure that the GR estimator have explicit solutions, which helps to deliver our messages concisely. Similar assumptions are commonly made in related literature \citep{lei2021near,wu2022power,li2023fixed}.
In Section \ref{sec:extension}, we will show that without these assumptions, similar results still hold.

By simple linear algebra, Assumption \ref{assump:commutable} is equivalent to the fact that $\{\boldsymbol{\Sigma}_t\}_{t=1}^T$ are simultaneously diagonalizable. 
Therefore, there exists a single orthogonal matrix $\boldsymbol{U}\in\sR^{p\times p}$ such that $\boldsymbol{\Sigma}_t = \boldsymbol{U} \boldsymbol{ \Gamma}_t \boldsymbol{U}^\top$, where $\boldsymbol{\Gamma}_t = \mathrm{diag}\{\gamma_1^{(t)},\dots,\gamma_p^{(t)}\}$ denotes the diagonal matrix consisting of the eigenvalues of $\boldsymbol{\Sigma}_t$.
In this case, the heterogeneity among different tasks is solely encoded by the different eigenvalues in $\boldsymbol{\Gamma}_t$. 

\begin{assumption}[Sufficient sample size]\label{assump:fullrank}
For each $j\in[p]$,
$\sum_{t=1}^T \gamma_j^{(t)}>0.$
\end{assumption}
This assumption is imposed to simplify the analysis of $\wh{\vw}^{(\mathrm{ORA})}$. 
Under this assumption, when the data of all $T$ tasks are pooled together, there is no overparameterization, i.e., $\sum_{t=1}^T \boldsymbol{\Sigma}_t$ has full rank. 
Therefore, the oracle estimator $\wh{\vw}^{\mathrm{(ORA)}}$ defined by \eqref{eq:ora} has a unique solution, whose estimation error can be calculated directly.

Indeed, the following lemma gives an explicit expression for the estimation error of the oracle estimator. 
\begin{lemma}[Estimation error of the oracle estimator]\label{lem:ora}
Suppose that Assumptions \ref{assump:fixed_design}--\ref{assump:fullrank} hold.
Then the estimator error of the oracle estimator is 
$$
\begin{aligned}
\gL(\wh{\vw}_T^{\mathrm{(ORA)}}) &= \sum_{j=1}^p \frac{\sigma^2}{\gamma_j^{(1)}n_1 + \dots + \gamma_j^{(t)}n_T}.
\end{aligned}
$$
\end{lemma}
As the task number $T$ increases, the estimation error of ORA is monotonically decreasing. Therefore, it does not suffer from the issue of catastrophic forgetting.

We remark that Assumption \ref{assump:fullrank} still allows a single task to be overparameterized.

\subsection{Main Results}

In this section, we consider a set of specific choices of $\boldsymbol{H}_t = \boldsymbol{U} \boldsymbol{\Lambda}_t \boldsymbol{U}^\top$, where $\boldsymbol{\Lambda}_t=\mathrm{diag}\{\lambda^{(t)}_1,\dots,\lambda^{(t)}_p\}$ is some diagonal matrix.
We show that if $\boldsymbol{\Lambda}_t$ is selected properly, the estimation error of GR is compatible with that of the oracle estimator; as a result, catastrophic forgetting is avoided.

We first decompose the estimation error into components along different directions.
Let $\vu_j\in\sR^p$ be the $j$th column of $\boldsymbol{U}$. 
Define $e_j^{(t)}:=(\vu_j^\top( \wh{\vw}_t^{\mathrm{(GR)}}-\vw_*))^2$ to be the projected estimation error of $\wh{\vw}_t^{\mathrm{(GR)}}$ onto $\vu_j$ for $j\ge 1$ and $e_j^{(0)}:=\left(\boldsymbol{u}_j^{\top}\boldsymbol{w}_*\right)^2$.
Since $\boldsymbol{U}$ is orthogonal, we have $\gL(\wh{\vw}_t^{(\mathrm{GR})})=\sum_{j=1}^p e_j^{(t)}$.

We are ready to present our main result regarding the estimation error of the GR estimator.

\begin{theorem}\label{thm:adaptive-ridge}
Suppose that Assumptions \ref{assump:fixed_design}--\ref{assump:fullrank} hold. Consider $\boldsymbol{H}_t = \boldsymbol{U} \boldsymbol{\Lambda}_t \boldsymbol{U}^\top$, where $\boldsymbol{\Lambda}_t=\mathrm{diag}\{\lambda^{(t)}_1,\dots,\lambda^{(t)}_p\}$ is some diagonal matrix. Then the projected estimation error satisfies
\begin{equation}\label{eq:error_iter}
\begin{aligned}
\E\left[e_j^{(t)}\right]
&= \E\left[e_j^{(t-1)}\right] - 2 \frac{\gamma_j^{(t)} \E\left[e_j^{(t-1)}\right]}{\lambda_j^{(t)} + \gamma_j^{(t)}}\\
&\mathrel{\phantom{=}}{}+ \frac{(\gamma_j^{(t)})^2 \E\left[e_j^{(t-1)}\right] + \gamma_j^{(t)} \sigma^2/n}{(\lambda_j^{(t)} + \gamma_j^{(t)})^2}.
\end{aligned}
\end{equation}
If we set $\boldsymbol{\Lambda}_t$ by
\begin{equation}\label{eq:lambda}
\lambda_j^{(t)} = \frac{\sigma^2/e_j^{(0)} + \gamma_j^{(1)}n_1 + \dots + \gamma_j^{(t-1)}n_{t-1} }{n_t}
\end{equation}
for each $j\in[p]$ and $t\in[T]$, then \eqref{eq:error_iter} is minimized and we have
$$
\E\left[e_j^{(t)}\right] = \frac{\sigma^2}{\sigma^2/e_j^{(0)} + \gamma_j^{(1)}n_1 + \dots + \gamma_j^{(t)}n_t},
$$
which further implies
\begin{equation}\label{eq:main-result}
\gL(\wh{\vw}_t^{(\mathrm{GR})}) = \sum_{j=1}^p \frac{\sigma^2}{\sigma^2/e_j^{(0)} + \gamma_j^{(1)}n_1 + \dots + \gamma_j^{(t)}n_t}.
\end{equation}
\end{theorem}
Under the choices of regularization weight matrices given in Theorem \ref{thm:adaptive-ridge}, we see that the estimation error of the GR estimator is monotonically nonincreasing with task index $t$. Indeed, as long as the covariance matrices $\boldsymbol{\Sigma}_t$ are positive definite, the estimation error is strictly decreasing. Therefore, the forgetting error $\gF_t\leq 0$ for every $t\in[T]$ and hence catastrophic forgetting is eliminated, even though we allow a single task to be overparameterized and the covariance matrices to be different across tasks.

Compared with Lemma \ref{lem:ora}, the estimation errors of the GR and oracle estimators are asymptotically equivalent as $T$ increases, even though the latter can only be calculated when pooling data of all tasks together. Indeed, the only difference between them is the additional term $\sigma^2 / e_j^{(0)}$ in the denominator of the estimation error of GR. Therefore, the estimation error of GR is even slightly smaller than that of the oracle estimator. This is because GR has an extra ridge term when learning the first task, whereas the oracle estimator has no regularization term. We also remark that given a fixed set of tasks, the final estimation error $\gL(\wh{\vw}_T^{(\mathrm{GR})})$ is independent of the task ordering, although the choice of $\boldsymbol{H}_t$ is dependent on it.

The key to achieving these desirable properties lies in the specific form of $\{\boldsymbol{H}_t\}_{t=1}^T$.
From the proof of Theorem \ref{thm:adaptive-ridge}, we identify two crucial considerations in choosing $\{\boldsymbol{H}_t\}_{t=1}^T$.
\begin{enumerate}[label=(\arabic*)]
\item The first consideration concerns balancing the trade-off between the information carried in $\wh{\vw}_{t-1}$ and that in $\gD_t$, i.e., the forward--backward trade-off. For example, if the estimation error of $\wh{\vw}_{t-1}$ is relatively small (larger $n_{\tau}$ for $\tau\leq t-1$) compared with the error of the new task, $\sigma^2/n_t$, we should increase the regularization strength $\lambda_j^{(t)}$.
\item The second one involves addressing the information heterogeneity among different tasks. As the covariance matrices vary, the amount of information pertaining to different directions of $\vw_*$ within different tasks may differ. Therefore, $\boldsymbol{\Lambda}_t$ should adapt to this information heterogeneity, allowing $\lambda_i^{(t)}$ and $\lambda_j^{(t)}$ to be different for $i\ne j$.
\end{enumerate}

The choice of hyperparameters specified in Theorem \ref{thm:adaptive-ridge} effectively addresses the forward--backward trade-off and information heterogeneity, thereby avoiding catastrophic forgetting and achieving an estimation error comparable with that of the oracle estimator.

We remark that Theorem \ref{thm:adaptive-ridge} does not necessitate $p<n_t$; it allows any individual task to be overparameterized. 
As long as aggregating all the data leads to an underparameterized linear regression problem, we can progressively improve the estimation of $\vw_*$ as new tasks are continuously introduced using generalized $\ell_2$-regularization. Ultimately, we achieve the error rate of the oracle estimator after completing the final task.

\subsection{A Practical Algorithm}\label{sec:practical-algorithm}

Now we take a closer look at the optimal choice of $\{\boldsymbol{H}_t\}_{t=1}^T$ developed in Theorem \ref{thm:adaptive-ridge}. Substituting \eqref{eq:lambda} into the definition of $\boldsymbol{H}_t$ gives 
$$
\boldsymbol{H}_t = \frac{1}{n_t} (n_1 \boldsymbol{\Sigma}_1 + \dots + n_{t-1} \boldsymbol{\Sigma}_{t-1} + \sigma^2\boldsymbol{U}\boldsymbol{E_0}\boldsymbol{U}^\top),
$$
which is the summation of the covariance matrices of old tasks weighted by sample sizes plus an additional error term. Tasks with larger sample size will be allocated with larger weights in the optimal regularization matrix, which is reasonable since they contain more information about $\vw_*$.

If $n_t$ is sufficiently large, the term $\sigma^2\boldsymbol{U}\boldsymbol{E_0}\boldsymbol{U}^\top/n_t$ in $\boldsymbol{H}_t$ is negligible and we can approximate $\boldsymbol{H}_t$ by
\begin{equation}\label{def:H}
\tilde{\boldsymbol{H}}_t := \frac{1}{n_t} (n_1 \boldsymbol{\Sigma}_1 + \dots + n_{t-1} \boldsymbol{\Sigma}_{t-1}) \approx \boldsymbol{H}_t,
\end{equation}
which can be easily computed in practice. This approximation makes the generalized $\ell_2$-regularized estimator a practical algorithm, which can be implemented without any underlying knowledge about the true parameter.

\paragraph{Connection with other regularization methods.}
Note that in linear regression, the covariance matrix is just the Hessian matrix (or Fisher information matrix) of the loss function. 
This links our GR estimator to some other popular regularization-based algorithms such as EWC and its variants \cite{kirkpatrick2017overcoming, huszar2018note, schwarz2018progress}. 
Specifically, if all tasks have the same sample size, our method recovers the online EWC proposed by \citet{schwarz2018progress} with the hyperparameter $\gamma = 1$.
Our theory gives a precise characterization of how to combine the Fisher information of old tasks properly in continual linear regression. 

\paragraph{Approximate weight matrices.}
We now present a result demonstrating that using the approximate optimal weight matrices has minimal impact on the estimation error when certain conditions are met. To this end, we define $\rho_j^{(t)}:=\gamma_j^{(t)} / (e_0^{(j)} + \gamma_j^{(1)}n_1 + \dots \gamma_j^{(t-1)}n_{t-1})$, which can be viewed as the information ratio between the new task $t$ and the old tasks. A larger $\rho$ indicates that the new task contains more information.

\begin{theorem}\label{thm:approx}
Suppose that Assumptions \ref{assump:fixed_design}--\ref{assump:fullrank} hold.
Assume that we use $\tilde{\boldsymbol{H}}_t:=\mU \tilde{\boldsymbol{\Lambda}}_t \mU^\top$ instead of $\boldsymbol{H}_t$ defined in Theorem \ref{thm:adaptive-ridge} as the regularization weight matrices.
Let $\Delta_j^{(t)}=1/(\wt{\lambda}_j + \gamma_j) - 1/(\lambda_j + \gamma_j)$.
Suppose that there exists some constant $C>0$ such that
\begin{equation}\label{eq:sensitivity}
(\gamma_j^{(t)} \Delta_j^{(t)})^2 \leq \frac{C(C-1)(\rho_j^{(t)})^2}{(1+\rho_j^{(t)})(1+C\rho_j^{(t)})^2}
\end{equation}
for each $t$. Then for every $t\in[T]$, we have
$$
\gL(\wh{\vw}_t^{(\mathrm{GR})}) \leq \frac{C}{e_j^{(0)} + \gamma_j^{(1)}n_1 + \dots + \gamma_j^{(t)} n_t}.
$$
\end{theorem}

Since $\rho_j^{(t)}$ is of order $o(1)$, the right-hand side of \eqref{eq:sensitivity} is roughly $O((\rho_{j}^{(t)})^2)$. 
Therefore, as $\rho_j^{(t)}$ becomes larger, which means that there is relatively more information of $\vu_j^\top \vw_*$ contained in the new task $t$, the requirement on the approximation accuracy of $\wt{\lambda}_j^{(t)}$ becomes looser. In this case, we can still attain the oracle rate without calculating the optimal regularization matrix very accurately. 
In Section \ref{sec:experiments}, we will conduct experiments to illustrate the performance of the generalized $\ell_2$-regularized estimator using $\tilde{\mH}_t$ defined in \eqref{def:H} instead of ${\mH}_t$.

\section{Connection Between Early Stopping and $\ell_2$-Regularization}\label{sec:earlystopping}
Besides adding a penalty term to the loss function, another commonly used regularization method is early stopping. 
When training a single task, several works \cite{raskutti2014early,ali2019continuous} have shown that applying gradient descent with early stopping is equivalent to ridge regression, in both classification and regression tasks. 
However, in continual learning where there is a sequence of tasks to be learned, similar results are still limited. 
In this section, we show that such equivalence also exists in continual linear regression.

Specifically, we formulate the early stopping estimator (ES) for continual linear regression in the following algorithm. Specifically, let $\wh{\vw}_0^{(\mathrm{ES})}=0$ be the initial value. At each task $t$, we set $\wh{\vw}_{t-1}^{(\mathrm{ES})}$ as the initial point and apply $m_t$-step gradient descent to the loss function of this new task, where $\boldsymbol{A}_t$ is a positive definite matrix used to control the learning rate and $m_t$ is the number of gradient descent iterations. 
\begin{algorithm}[htbp]\label{alg:earlystopping}
\caption{Early stopping estimator}
\begin{algorithmic}
\STATE {\bfseries Initialization:} $\wh{\vw}_0^{\mathrm{(ES)}} = 0$
\FOR{\textbf{each task} $t=1$ {\bfseries to} $T$}
\STATE $\vw_t^{(0)} = \wh{\vw}_{t-1}^{(\mathrm{ES})};$
\FOR{$\tau=1$ {\bfseries to} $m_t$}
\STATE $\quad \vw_t^{(\tau)} = \vw_t^{(\tau-1)} - (\boldsymbol{A}_t/n) \boldsymbol{X}_t^\top (\boldsymbol{X}_t \vw_t^{(\tau-1)} - \boldsymbol{y}_t);$
\ENDFOR
\STATE $\wh{\vw}_{t}^{(\mathrm{ES})} = \vw_t^{(m_t)};$
\ENDFOR
\end{algorithmic}
\end{algorithm}

Note that in ordinary gradient descent, $\boldsymbol{A}_t$ is simply $s_t \boldsymbol{I}_p$ for some $s_t>0$, which we refer to as \textit{vanilla early stopping} (vanilla ES). In contrast, here we take a more general form of the learning rate matrix in order to capture the information heterogeneity and align with the generalized $\ell_2$-regularization studied above.

The following theorem establishes the equivalence between the ES and GR estimators.

\begin{theorem}\label{thm:earlystopping}
Assume that $\boldsymbol{\Sigma}_t = \boldsymbol{U}_t \boldsymbol{\Gamma}_t \boldsymbol{U}_t^\top$, $\boldsymbol{A}_t = \boldsymbol{U}_t \boldsymbol{S}_t \boldsymbol{U}_t^\top$, and $\boldsymbol{H}_t = \boldsymbol{U}_t \boldsymbol{\Lambda}_t \boldsymbol{U}_t^\top$ for some positive definite diagonal matrices $\boldsymbol{\Gamma}_t = \mathrm{diag}\{\gamma_1^{(t)},\dots,\gamma_p^{(t)}\}$, $\boldsymbol{S}_t = \mathrm{diag}\{s_1^{(t)},\dots,s_p^{(t)}\}$ and $\boldsymbol{\Lambda}_t = \mathrm{diag}\{\lambda_1^{(t)},\dots,\lambda_p^{(t)}\}$ satisfying
\begin{equation}\label{eq:equivalence}
\lambda_j^{(t)} = \frac{\gamma_j^{(t)} (1-s_j^{(t)} \gamma_j^{(t)})^{m_t}}{1-(1-s_j^{(t)} \gamma_j^{(t)})^{m_t}}
\end{equation}
for each $j\in[p]$ and $t\in[T]$.
Then for each $t\in[T]$, we have
$$
\wh{\vw}_t^{(\mathrm{ES})} = \wh{\vw}_t^{(\mathrm{GR})},
$$
where $\wh{\vw}_t^{(\mathrm{ES})}$ is the ES estimator using the learning rate matrix $\boldsymbol{A}_t$, and $\wh{\vw}_t^{(\mathrm{GR})}$ is the GR estimator using the regularization weight matrix $\boldsymbol{H}_t$.
\end{theorem}

Note that this result does not require commutable covaraince matrices in Assumption \ref{assump:commutable}. From Theorem \ref{thm:earlystopping} we conclude that with some proper choices of the learning rate matrix $\boldsymbol{A}_t$ and regularization weight matrix $\boldsymbol{H}_t$, the ES estimator $\wh{\vw}_t^{(\mathrm{ES})}$ and the GR estimator $\wh{\vw}_t^{(\mathrm{GR})}$ output exactly the same estimates for each $t$. Indeed, the errors $\wh{\vw}_t - \vw_*$ of these two estimators are both the weighted average of the error of the $(t-1)$th task $\wh{\vw}_{t-1}-\vw_*$ and the variance term for the new task, $\mX_t^\top\mvarepsilon_t/n$, where the weights are determined by the learning rate matrix $\boldsymbol{A}_t$, iteration number $m_t$, and regularization weight matrix $\boldsymbol{H}_t$.

We remark that \eqref{eq:equivalence} is required to hold for each $j\in[p]$.
Therefore, vanilla ES with $\mA_t=s_t\mI_p$ and vanilla $\ell_2$-regularization with $\mH_t=\lambda_t\mI_p$ may not be equivalent since $\gamma_j^{(t)}$ could be different for different $j$ and a single $\lambda_t$ and $s_t$ could not make \eqref{eq:equivalence} hold for every $j\in[p]$.
It could happen that vanilla ES is equivalent to some generalized $\ell_2$-regularized estimator or vice verse.

Similar to $\ell_2$-regularization, early stopping with proper learning rate matrix $\mA_t$ can also avoid catastrophic forgetting and attain the oracle rate. 

\begin{corollary}\label{coro:es}
Suppose that Assumption \ref{assump:fixed_design}--\ref{assump:fullrank} hold. Assume that $\boldsymbol{A}_t =\boldsymbol{U}\boldsymbol{S}_t \boldsymbol{U}^\top$ for some diagonal matrix $\boldsymbol{S}_t = \mathrm{diag}\{s_1^{(t)},\dots,s_p^{(t)}\}$ satisfying
\begin{equation}\label{eq:eses}
\left(1-s_j^{(t)} \gamma_j^{(t)}\right)^{m_t} = 1- \frac{\gamma_j^{(t)} n_t}{\sigma^2/e_j^{(0)} + \gamma_j^{(1)} n_1 + \dots + \gamma_j^{(t)} n_t}
\end{equation}
for each $j\in[p]$. 
Then the estimation error of $\wh{\vw}_t^{(\mathrm{ES})}$ is 
$$
\gL(\wh{\vw}_t^{(\mathrm{ES})}) = \sum_{j=1}^p \frac{\sigma^2}{\sigma^2/e_j^{(0)} + \gamma_j^{(1)}n_1 + \dots + \gamma_j^{(t)}n_t}.
$$
\end{corollary}

If the new task $t$ has a larger sample size $n_t$, the term $(1-s_j\gamma_j)^{m_t}$ should decrease by \eqref{eq:eses}, implying that both the learning rate $s_j^{(t)}$ and the iteration number $m_t$ should be increased. 
This means that when task $t$ provides more information, we should traverse a more extensive path in the gradient descent process, allowing for a deeper utilization of the new data.

\section{Extensions}\label{sec:extension}

In this section, we discuss some possible extensions to relax our model assumptions.

\paragraph{Commutable covariance matrices.}
The main purpose of Assumption \ref{assump:commutable} is to obtain explicit forms for some crucial quantities of $\ell_2$-regularized estimators, such as the optimal regularization matrix $\mH_t$ and the corresponding optimal estimation error.

Without this assumption, even though the optimal estimation error does not have an explicit form, we can still show that there exist some regularization weight matrices such that the estimation error is monotonically nonincreasing with $t$. Therefore, catastrophic forgetting can still be avoided.

Specifically, we have the following result without Assumption \ref{assump:commutable}.

\begin{theorem}\label{thm:extension}
There exist $\{\lambda_j^{(t)},j=1,\dots,p,t=1\dots,T\}$ such that for each $t\in[T]$,
$$
\gL(\wh{\vw}_t^{(\mathrm{GR})}) \leq \gL(\wh{\vw}_{t-1}^{(\mathrm{GR})}),
$$
where the strict inequality holds for $t$ satisfying $\sum_{j=1}^p \gamma_j^{(t)} > 0$.
\end{theorem}

Intuitively, the condition $\sum_{j=1}^p \gamma_j^{(t)}$ means that task $t$ has nonzero information about the true parameter $\vw_*$. Therefore, there always exists some choice of the regularization weight matrix under which we can leverage the new information and improve on the existing estimator. 

Moreover, in Section \ref{sec:experiments} we will empirically show that violating this assumption will not cause a significant performance degradation for our method.

\paragraph{Other loss functions.}
Our theory can be extended to general convex loss functions. In this scenario, the Hessian matrix of the loss function at the true parameter plays the role of the data covariance matrix in linear regression. The heterogeneity among different tasks is encoded by the differences in the Hessian matrices. Our analysis can then proceed with some modifications.

\paragraph{Common true parameters.}
Our model \eqref{eq:data} assumes that all tasks share the same true parameter $\vw_*$. In real-world continual learning, new challenges may arise when the true parameters are different across tasks. Analyzing the setting with distinct true parameters may need to introduce more trade-offs and insights. For example, if the parameters are not too far apart, our results may still hold with an additional error term. On the other hand, if the parameters differ significantly, negative transfer may dominate and continual learning might not work at all. We leave a comprehensive analysis of this problem to future work.

\section{Experiments}\label{sec:experiments}

We conduct simulation experiments to illustrate the performance of continual ridge regression (CRR), the minimum norm estimator (MN), and the generalized $\ell_2$-regularized estimator (GR).

\paragraph{Data generation.}
We consider two data generating settings, namely with and without covariate shift. The difference between them is whether the covariance matrices are the same for different tasks.
\begin{enumerate}[label=(\arabic*)]
\item \textit{Without covariate shift.}
The true parameter $\vw_*$ is sampled from $\gN(0, \mI_p)$ and is fixed for each task.
The features $\vx_i^{(t)}$ are independently sampled from $\gN(0, \mI_p)$ and the noises ${\varepsilon}_i^{(t)}$ are independently sampled from $\gN(0, \sigma^2)$. 
Then the labels are generated by $y_i^{(t)}=\vw_*^\top \vx_i^{(t)} + \varepsilon_i^{(t)}$.
\item \textit{With covariate shift.}
The true parameter $\vw_*$ is sampled from $\gN(0, \mI_p)$ and is fixed for each task. The covariance matrices of the features are generated as follows. We first randomly sample the eigenvalues $\gamma_t^{(j)}$ by $P(\gamma_t^{(j)}=1)=0.99$ and $P(\gamma_t^{(j)}=100)=0.01$. Then the covariance matrices are set by $\boldsymbol{\Sigma}_t:=\mathrm{diag}\{\gamma_t^{(1)}, \dots, \gamma_t^{(p)}\}$. 
After the covariance matrices are generated, the features $\vx_i^{(t)}$ are independently sampled from $\gN(0, \boldsymbol{\Sigma}_t)$ and the noises ${\varepsilon}_i^{(t)}$ are independently sampled from $\gN(0, \sigma^2)$. 
Finally, the labels are generated by $y_i^{(t)}=\vw_*^\top \vx_i^{(t)} + \varepsilon_i^{(t)}$.
\end{enumerate}

\paragraph{Experimental configuration.}
We compare the performance of the CRR, MN, and GR algorithms with that of the oracle estimator. The regularization weight matrices of GR are set to $\tilde{\mH}_t$ as discussed in Section \ref{sec:practical-algorithm}.

We set the task number $T=20$ and sample size $n_1=\dots=n_t=150$. The parameter dimension $p=200$, and hence each single task is overparameterized.
We consider two noise levels: $\sigma^2=1$ or $5$.
We repeated our experiments $100$ times and present the average results.

\paragraph{Simulation results.}

\begin{figure}[ht]
\centering
\includegraphics[width=1\linewidth]{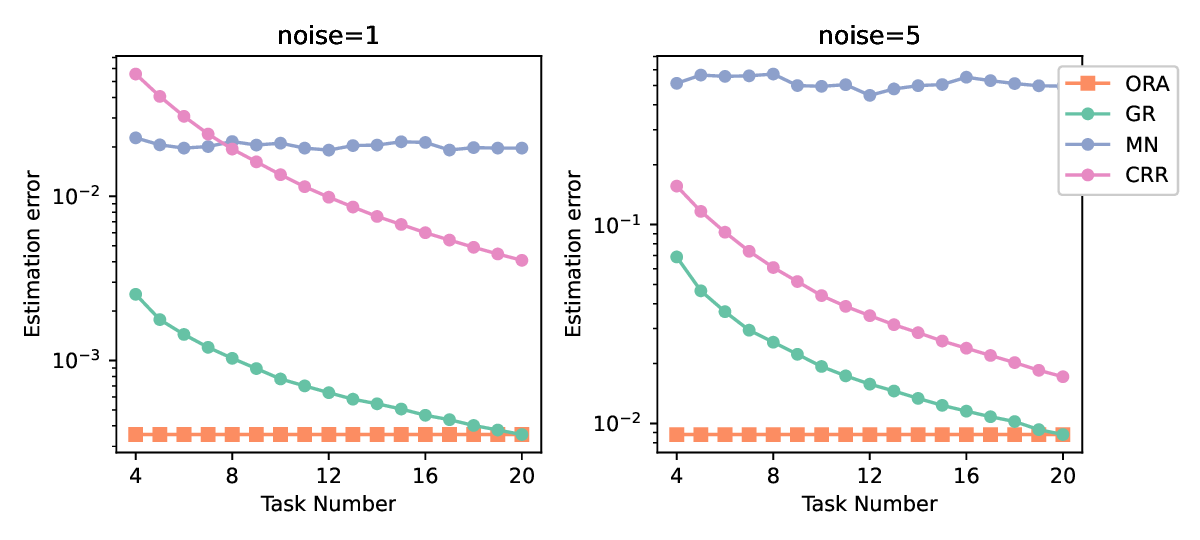}
\caption{Simulation results for different noise levels: $T=20$, $n_t=150$, $p=200$, $\sigma^2=1$ or $5$, and no covariate shift.}
\label{fig:1}
\end{figure}

The simulation results for different noise levels are depicted in Figure \ref{fig:1}. We observe that the estimation error of the MN estimator remains nearly constant as the task number $t$ increases. Furthermore, a higher noise level makes the MN estimator perform worse than the other methods. This highlights the sensitivity of MN to noise.

\begin{figure}[ht]
\centering
\includegraphics[width=1\linewidth]{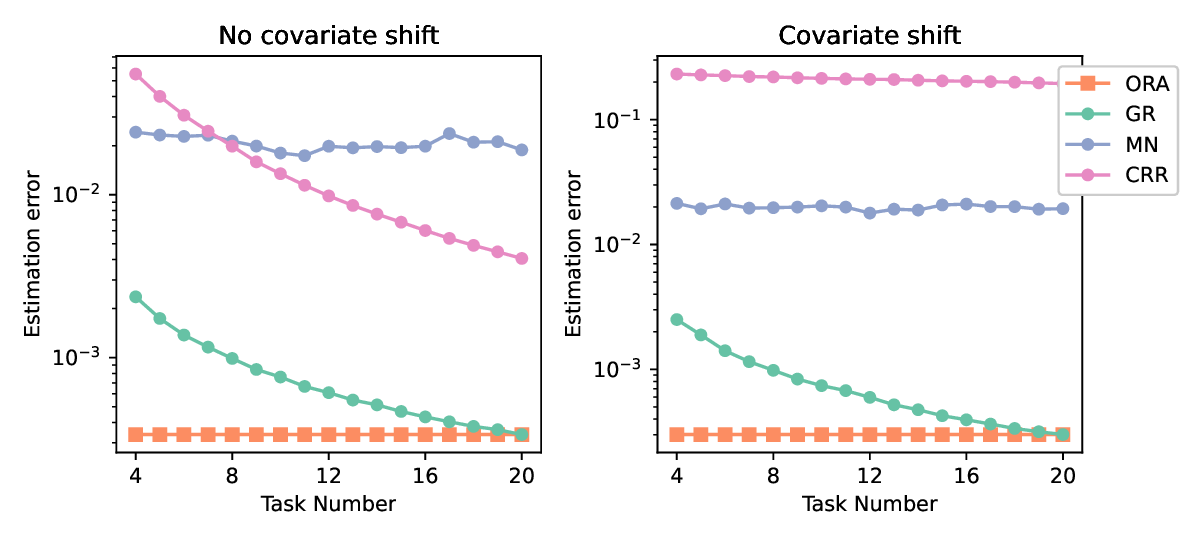}
\caption{Simulation results with and without covariate shift: $T=20$, $n_t=150$, $p=200$, and $\sigma^2=1$.}
\label{fig:2}
\end{figure}

The simulation results with and without covariate shift are contrasted in Figure \ref{fig:2}, from which we find that covariate shift makes the CRR estimator worse. In the absence of covariate shift, CRR exhibits a decreasing loss, even though it is inferior to the GR estimator. In the presence of covariate shift, the performance of CRR deteriorates significantly, and its estimation error remains approximately constant.

In each case, the GR estimator consistently demonstrates a decreasing estimation error, which eventually converges to the oracle estimator. It is noteworthy that, due to the random generating process for sampling the features, Assumption \ref{assump:commutable} does not hold for the empirical covariance matrices. Nevertheless, this departure does not adversely impact the performance of our method.

\section{Conclusion}
Our analysis focuses on regularization-based continual learning across a series of linear regression tasks. We establish the estimation error of the oracle estimator with access to all data concurrently. We then explore a set of generalized $\ell_2$-regularization algorithms characterized by matrix-valued hyperparameters. We develop an iterative formula to update the estimation error for these generalized $\ell_2$-regularized estimators when new tasks are introduced. This allows us to identify the hyperparameters that optimize the performance of the algorithm. Remarkably, selecting the optimal hyperparameters achieves a balanced trade-off between forward and backward knowledge transfer and accommodates the variability in data distribution. Furthermore, we explicitly derive the estimation error of the optimal algorithm, which is found to match the order for the oracle estimator. Finally, we show that early stopping and generalized $\ell_2$-regularization, rather than the conventional ridge regression, are equivalent in the context of continual learning, thereby addressing a question raised by \citet{evron2023continual} on the connection between early stopping and explicit regularization in continual learning.

\section*{Acknowledgments}
Xuyang Zhao and Wei Lin are supported by National Natural Science Foundation of China grants 12171012, 12292980, and 12292981. Weiran Huang is supported by 2023 CCF-Baidu Open Fund and Microsoft Research Asia. We thank three reviewers for valuable comments that have helped improve the paper.

\section*{Impact Statement}
This paper presents work whose goal is to advance the field of
Machine Learning. There are many potential societal consequences
of our work, none of which we feel must be specifically highlighted here.

\bibliography{example_paper}
\bibliographystyle{icml2024}

\newpage
\appendix
\onecolumn

\clearpage

\appendix
\begin{center}
\Large\bf Appendix
\end{center}
\section{Proofs for Section \ref{sec:lower}}

\begin{proof}[Proof of Theorem \ref{thm:min-norm}]
By the definition of minimum norm estimator, we have
$$
\boldsymbol{X}_t \boldsymbol{\wh{w}}^{(\mathrm{MN})}_t = \boldsymbol{y}_t
$$
for each $t\in[T]$.
Therefore, 
$$
\boldsymbol{X}_t(\boldsymbol{\wh{w}}^{\mathrm{MN}}_t - \boldsymbol{w}_*) = \boldsymbol{y}_t - \boldsymbol{X}_t \boldsymbol{w}_* = \boldsymbol{\varepsilon}_t,
$$
which implies
$$
(\wh{\vw}^{(\mathrm{MN})}_t - \vw_*)^\top \boldsymbol{\Sigma}_t (\wh{\vw}^{\mathrm{MN}}_t - \vw_*) = \frac{1}{n_t}\|\boldsymbol{X}_t(\wh{\vw}_t - \vw_*)\|^2 = \frac{1}{n_t}\|\boldsymbol{\varepsilon}_t\|^2.
$$
Taking expectation with respect to $\boldsymbol{\varepsilon}_t$ on both sides gives 
$$
E(\wh{\vw}^{(\mathrm{MN})}_t - \vw_*)^\top \boldsymbol{\Sigma}_t (\wh{\vw}^{(\mathrm{MN})}_t - \vw_*) = \boldsymbol{\sigma}^2.
$$
By the property of eigenvalues, we have
$$
E(\wh{\vw}_t^{(\mathrm{MN})} - \vw_*)^\top \boldsymbol{\Sigma}_t (\wh{\vw}_t^{(\mathrm{MN})} - \vw_*) \leq \max_{j\in[p]} \gamma_j^{(t)} E\|\wh{\vw}_t - \vw_*\|^2.
$$
Therefore, we finally conclude that
$$
E\|\wh{\vw}_t - \vw_*\|^2 \geq \frac{\sigma^2}{\max_{j\in[p]} \gamma_j^{(t)}}.
$$
\end{proof}

\begin{proof}[Proof of Theorem \ref{thm:ridge}]
Without loss of generality, we assume $\vw_{*,1}^2 = \vw_{*,1}^2 = 1$.
In this two-task problem, the definition of CRR estimator is 
$$
\wh{\vw}_1 = \argmin_{\vw} \left\{\frac{1}{n_1} \|\mX_1\vw - \vy_1\|^2 + \lambda_1 \|\vw\|^2\right\},
$$
$$
\wh{\vw}_2 = \argmin_{\vw} \left\{\frac{1}{n_2} \|\mX_2\vw - \vy_2\|^2 + \lambda_2 \|\vw-\wh{\vw}_1\|^2\right\},
$$
where $\lambda_1$ and $\lambda_2$ are the hyperparameters.
\paragraph{Task 1}
By taking derivatives, we can explicitly obtain the solution of $\wh{\vw}_1$:
$$
\wh{\vw}_1 = (\mSigma_1 + \lambda_1\mathbf{I}_2)^{-1}(\mX_1 \vy_1 / n_1).
$$
By the definition of $\mSigma_1$, we further have
$$
\wh{\vw}_{1,1} = \frac{1}{1+\lambda_1} \frac{\mX_{1,1}^\top\vy_1}{n_1}
$$
and
$$
\wh{\vw}_{1,2} = \frac{1}{\epsilon+\lambda_1} \frac{\mX_{1,2}^\top\vy_1}{n_1},
$$
where $\wh{\vw}_{1,j}$ is the $j$th coordinate of $\wh{\vw}$ and $\mX_{1,j}$ is the $j$th column of $\mX$. 
Therefore, using the definition of $\mSigma_1$ again we obtain
$$
\begin{aligned}
\E\, (\wh{\vw}_{1,1} - \vw_{*,1})^2 &= \E\,\left(\frac{1}{1+\lambda_1}\frac{\mX_{1,1}^\top (\mX_1 \vw_* + \boldsymbol{\varepsilon}_1)}{n_1} - \vw_{*,1}\right)^2 \\
&= \E\,\left(\frac{\vw_{*,1}}{1+\lambda_1} - \vw_{*,1} + \frac{1}{1+\lambda_1}\frac{\mX_{1,1}^\top \boldsymbol{\varepsilon}_1}{n_1}\right)^2 \\
&= \left(\frac{\lambda_1}{1+\lambda_1}\right)^2 + \left(\frac{1}{1+\lambda_1}\right)^2 \frac{\sigma^2}{n_1} \\
&\geq \frac{\sigma^2}{n_1 + \sigma^2},
\end{aligned}
$$
where the last equation holds if and only if $\lambda_1 = \sigma^2 / n_1$.
Similarly, for $\wh{\vw}_{1,2}$ we have
$$
\begin{aligned}
\E\, (\wh{\vw}_{1,2} - \vw_{*,2})^2 &= \E\,\left(\frac{1}{\epsilon+\lambda_1}\frac{\mX_{1,2}^\top (\mX_1 \vw_* + \boldsymbol{\varepsilon}_1)}{n_1} - \vw_{*,2}\right)^2 \\
&= \E\,\left(\frac{\epsilon}{\epsilon+\lambda_1}\vw_{*,2} - \vw_{*,2} + \frac{1}{\epsilon+\lambda_1}\frac{\mX_{1,2}^\top \boldsymbol{\varepsilon}_1}{n_1}\right)^2 \\
&= \left(\frac{\lambda_1}{\epsilon+\lambda_1}\right)^2 +\frac{\epsilon}{(\epsilon+\lambda_1)^2} \frac{\sigma^2}{n_1} \\
&\geq \frac{\sigma^2}{\epsilon n_1 + \sigma^2},
\end{aligned}
$$
where the last equation holds if and only if $\lambda_1 = \sigma^2/\epsilon n_1$.

\paragraph{Task 2} Through almost the same analysis, for $\wh{\vw}_2$ we have
$$
\begin{aligned}
\E\left[(\wh{\vw}_{2,1} - \vw_{*,1})^2|\wh{\vw}_1\right] &= \E\,\left[\frac{1}{\epsilon+\lambda_2}\left(\frac{\mX_{2,1}^\top (\mX_2 \vw_* + \boldsymbol{\varepsilon}_2)}{n_2} + \lambda_2 \wh{\vw}_{1,1}\right) - \vw_{*,1}\right]^2 \\
&= \left(\frac{\lambda_2}{\epsilon+\lambda_2}\right)^2(\wh{\vw}_{1,1} - \vw_{*,1})^2 +\frac{\epsilon}{(\epsilon+\lambda_2)^2} \frac{\sigma^2}{n_2}
\end{aligned}
$$
and 
$$
\begin{aligned}
\E\left[(\wh{\vw}_{2,2} - \vw_{*,2})^2|\wh{\vw}_1\right] &= \E\,\left[\frac{1}{1+\lambda_2}\left(\frac{\mX_{2,2}^\top (\mX_2 \vw_* + \boldsymbol{\varepsilon}_2)}{n_2} + \lambda_2 \wh{\vw}_{1,2}\right) - \vw_{*,2}\right]^2 \\
&= \left(\frac{\lambda_2}{1+\lambda_2}\right)^2(\wh{\vw}_{1,2} - \vw_{*,2})^2 +\frac{1}{(1+\lambda_2)^2} \frac{\sigma^2}{n_2}.
\end{aligned}
$$

From Theorem \ref{thm:adaptive-ridge}, we know that
$$
\E(\wh{\vw}^{\text{(GR)}}_{2,1} - \vw_{*,1})^2 = O\left(\frac{\sigma^2}{n_1 + \epsilon n_2}\right)
$$
and 
$$
\E(\wh{\vw}^{\text{(GR)}}_{2,2} - \vw_{*,2})^2 = O\left(\frac{\sigma^2}{\epsilon n_1 + n_2}\right).
$$

By some calculations, if 
$$
\frac{\E(\wh{\vw}^{\text{(CRR)}}_{2,1} - \vw_{*,1})^2}{\E(\wh{\vw}^{\text{(GR)}}_{2,1} - \vw_{*,1})^2} < \infty
$$
when $n_1, n_2 \to \infty$ and $\epsilon\to 0$, we need
$$
\epsilon\frac{1-\sqrt{\frac{\epsilon n_2}{n_1 + \epsilon n_2}}}{\sqrt{\frac{\epsilon n_2}{n_1 + \epsilon n_2}}}\lesssim \lambda_2 \lesssim \epsilon \frac{\sqrt{\frac{n_1}{n_1 + \epsilon n_2}}}{1-\sqrt{\frac{n_1}{n_1 + \epsilon n_2}}}.
$$
If 
$$
\frac{\E(\wh{\vw}^{\text{(CRR)}}_{2,2} - \vw_{*,2})^2}{\E(\wh{\vw}^{\text{(GR)}}_{2,2} - \vw_{*,2})^2} < \infty
$$
when $n_1, n_2 \to \infty$ and $\epsilon\to 0$, we need
$$
\sqrt{\frac{\epsilon n_1 + n_2}{n_2}} - 1 \lesssim \lambda_2 \lesssim \frac{\sqrt{\frac{1}{\epsilon n_1 + n_2}}}{1-\sqrt{\frac{1}{\epsilon n_1 + n_2}}}.
$$ 

We consider a special case, where $\epsilon n_2 / n_1 \to \infty$. Then the above requirements become
$$
\epsilon \frac{n_1}{n_1 + \epsilon n_2} \lesssim \lambda_2 \lesssim \epsilon \sqrt{\frac{n_1}{n_1 + \epsilon n_2}}
$$
and
$$
\frac{\epsilon n_1}{n_2} \lesssim \lambda_2 \lesssim \sqrt{\frac{1}{\epsilon n_1 + n_2}}.
$$
However, if $n_1=O(n^2)$, $n_2=O(n^3)$ and $\epsilon=O(n^{-0.5})$ for some $n\to\infty$, the lower bound of the first inequality if greater than the upper bound of the second inequality:
$$
\epsilon \frac{n_1}{n_1 + \epsilon n_2} = O(n^{-1})
$$
while 
$$
\sqrt{\frac{1}{\epsilon n_1 + n_2}} = O(n^{-1.5}).
$$

Therefore, by contradiction we have
$$
\sup_{n_1,n_2,\epsilon} \inf_{\lambda_2} \frac{\gL(\wh{\vw}_2^{(\text{CRR})})}{\gL(\wh{\vw}_2^{(\text{GR})})} = \infty.
$$
\end{proof}

\section{Proofs for Section \ref{sec:main}}
\begin{proof}[Proof of Lemma \ref{lem:ora}]
The oracle estimator $\wh{\vw}_T^{(\mathrm{ORA})}$ satisfies
$$
\sum_{t=1}^T X_t^\top(\boldsymbol{X}_t \wh{\vw}_T^{(\mathrm{ORA})} - \vy_t) =  0,
$$
which implies 
$$
\left(\sum_{t=1}^T \boldsymbol{X}_t^\top\boldsymbol{X}_t\right) \wh{\vw}_T^{(\mathrm{ORA})} = \sum_{t=1}^T \boldsymbol{X}_t^\top \vy_t.
$$
By Assumption \ref{assump:commutable} and \ref{assump:fullrank}, we have that $\sum_{t=1}^T \boldsymbol{X}_t^\top \boldsymbol{X}_t$ is invertible. Therefore, the ORA has the following explicit form solution
$$
\begin{aligned}
\wh{\vw}_T^{(\mathrm{ORA})} &= \left(\sum_{t=1}^T \boldsymbol{X}_t^\top \boldsymbol{X}_t\right)^{-1} \left(\sum_{t=1}^T \boldsymbol{X}_t^\top \vy_t\right) \\
&= \left(\sum_{t=1}^T \boldsymbol{X}_t^\top \boldsymbol{X}_t\right)^{-1} \left(\sum_{t=1}^T \boldsymbol{X}_t^\top (\boldsymbol{X}_t \vw_* + \boldsymbol{\varepsilon}_t)\right) \\
&= \vw_* + \left(\sum_{t=1}^T \boldsymbol{X}_t^\top \boldsymbol{X}_t\right)^{-1} \left(\sum_{t=1}^T \boldsymbol{X}_t^\top \boldsymbol{\varepsilon}_t\right).
\end{aligned}
$$
Taking expectation with respect to $\{\boldsymbol{\varepsilon}_t\}_{t=1}^T$, we obtain
$$
\begin{aligned}
\E\,\|\wh{\vw}_T^{(\mathrm{ora})} - \vw_*\|^2 &= E\left\|\left(\sum_{t=1}^T \boldsymbol{X}_t^\top \boldsymbol{X}_t\right)^{-1} \left(\sum_{t=1}^T \boldsymbol{X}_t^\top \boldsymbol{\varepsilon}_t\right)\right\|^2 \\
&= tr\left\{\left(\sum_{t=1}^T \boldsymbol{X}_t^\top \boldsymbol{X}_t\right)^{-1}\right\}.
\end{aligned}
$$
By Assumption \ref{assump:commutable}, we can further have
$$
\begin{aligned}
\tr\left\{\left(\sum_{t=1}^T \boldsymbol{X}_t^\top \boldsymbol{X}_t\right)^{-1}\right\} &= \tr\left\{\left(\sum_{t=1}^T n_t \boldsymbol{\Gamma}_t \right)^{-1}\right\} \\
&= \sum_{j=1}^p \frac{1}{n_1\gamma_j^{(1)} + \dots n_T \gamma_j^{(T)}},
\end{aligned}
$$
which completes the proof.

\end{proof}

\begin{proof}[Proof of Theorem \ref{thm:adaptive-ridge}]
For each $t=1,\dots,T$, the solution of GR estimator $\wh{\vw}_t^{(\mathrm{GR})}$ satisfies
$$
\frac{1}{n_t} \boldsymbol{X}_t^\top (\boldsymbol{X}_t \wh{\vw}_t^{(\mathrm{GR})}-\vy_t) + \boldsymbol{H}_t (\wh{\vw}_t^{(\mathrm{GR})}-\wh{\vw}_{t-1}^{(\mathrm{GR})}) = 0
$$
Since $\vy_t = \boldsymbol{X}_t \vw_* + \boldsymbol{\varepsilon}_t$, it can be written explicitly as
$$
\wh{\vw}_t^{(\mathrm{GR})} = \vw_* + (\boldsymbol{X}_t^\top \boldsymbol{X}_t + n \boldsymbol{H}_t)^{-1} \boldsymbol{X}_t^\top \boldsymbol{\varepsilon}_t + (\boldsymbol{X}_t^\top \boldsymbol{X}_t + n_t \boldsymbol{H}_t)^{-1} n_t \boldsymbol{H}_t (\wh{\vw}_{t-1}^{(\mathrm{GR})} - \vw_*).
$$
Therefore, for each $j=1,\dots,p$,
$$
\begin{aligned}
\vu_j^\top (\wh{\vw}_t^{(\mathrm{GR})}-\vw_*) &= \vu_j^\top (n_t\boldsymbol{\Sigma}_t + n_t \boldsymbol{H}_t)^{-1} \boldsymbol{X}_t^\top \boldsymbol{\varepsilon}_t + (n_t \boldsymbol{\Sigma}_t + n_t \boldsymbol{H}_t)^{-1} n_t \boldsymbol{\Lambda}_t \boldsymbol{U}^\top (\wh{\vw}_{t-1}^{(\mathrm{GR})} - \vw_*),
\end{aligned}
$$
which implies that

\begin{equation}\label{eq:main}
\begin{aligned}
\E\left[e_j^{(t)}\right] &= \E\left[\left\|\vu_j^\top(\wh{\vw}_t^{(\mathrm{GR})} - \vw_*)\right\|^2\right] \\
&\overset{(\mathrm{i})}{=} \E\left[(\vu_j^\top (\boldsymbol{X}_t^\top \boldsymbol{X}_t + n_t \boldsymbol{H}_t)^{-1} \boldsymbol{X}_t \boldsymbol{\varepsilon}_t)^2\right] + \E\,\left[(\vu_j^\top (\boldsymbol{X}_t^\top \boldsymbol{X}_t + n_t \boldsymbol{H}_t)^{-1} n_t \boldsymbol{H}_t (\wh{\vw}_{t-1}-\vw_*))^2\right] \\
&\overset{(\mathrm{ii})}{=} \E \left[(\vu_j^\top \mU(n_t\boldsymbol{\Gamma}_t + n_t\boldsymbol{\Lambda}_t)^{-1} \mU^\top \mX_t \mvarepsilon_t)^2\right] + n_t^2 \E\left[ (\vu_j^\top \mU(n_t \boldsymbol{\Gamma}_t + n_t\boldsymbol{\Lambda}_t)^{-1} \boldsymbol{\Lambda} \mU^\top (\wh{\vw}_{t-1}-\vw_*))^2\right] \\
&\overset{(\mathrm{iii})}{=} \frac{\gamma_j^{(t)} \sigma^2 /n + (\lambda_j^{(t)})^2 \E\left[e_j^{(t-1)}\right]}{(\lambda_j^{(t)} + \gamma_j^{(t)})^2} \\
&= \frac{\gamma_j^{(t)} \sigma^2 /n + (\lambda_j^{(t)} + \gamma_j^{(t)} - \gamma_j^{(t)})^2 \E\left[e_j^{(t-1)}\right]}{(\lambda_j^{(t)} + \gamma_j^{(t)})^2} \\
&= \E\left[e_j^{(t-1)}\right] - 2 \frac{\gamma_j^{(t)} \E\left[e_j^{(t-1)}\right]}{\lambda_j^{(t)} + \gamma_j^{(t)}} + \frac{(\gamma_j^{(t)})^2 \E\left[e_j^{(t-1)}\right] + \gamma_j^{(t)} \sigma^2/n}{(\lambda_j^{(t)} + \gamma_j^{(t)})^2},
\end{aligned}
\end{equation}
where (i) comes from the independence between $\mvarepsilon_t$ and $\wh{\vw}_{t-1}^{(\mathrm{GR})}$, (ii) comes from Assumption \ref{assump:commutable} and (iii) is obtained by the property of eigenvalues and eigenvectors. 
To derive the optimal value of $\lambda_j^{(t)}$
Now we consider two different cases: 
\begin{enumerate}
\item Consider the case $\gamma_j = 0.$ Then as long as $\lambda_j > 0$, $\boldsymbol{\Sigma}_t + \mH_t$ is invertible and we have
$$
\E\left[e_j^{(t)}\right] = \E\left[e_j^{(t-1)}\right].
$$
This means that data of task $t$ do not bring new information about the direction $j$ of $\vw_*$, which makes the $j$th projected error $e_j^{(t)}$ unchanged.

\item Consider the case $\gamma_j > 0$.
In this case, the last formula of Equation \ref{eq:main} can be regarded as a quadratic function of $1/(\lambda_j^{(t)}+\gamma_j^{(t)})$ as $\lambda_j^{(t)}$ changes.  
Therefore, the optimal $\lambda_j$ is obtained by 
$$
\frac{1}{\lambda_j^{(t)} + \gamma_j^{(t)}} = \frac{\E\left[e_j^{(t-1)}\right]}{\gamma_j^{(t)} \E\left[e_j^{(t-1)}\right] + \sigma^2/n_t},
$$
which is the minimum of the quadratic function. This further implies that $\lambda_j^{(t)} = \frac{\sigma^2/n_t}{\E(e_j^{(t-1)})^2}$,
where we have
\begin{equation}\label{eq:iteration}
\begin{aligned}
\E\left[e_j^{(t)}\right] &= \frac{\left(E\left[e_j^{(t-1)}\right]\right)^2 \gamma_j^2+ \E\left[e_j^{(t-1)}\right] \gamma_j \sigma^2 / n_t - \gamma_j^2 \left(\E\left[e_j^{(t-1)}\right]\right)^2}{\gamma_j^2 \left(\E\left[e_j^{(t-1)}\right]\right)^2 + \gamma_j \sigma^2 / n_t} \\
&= \frac{\E\left[e_j^{(t-1)}\right]\cdot \sigma^2 / (\gamma_j n_t)}{\E\left[e_j^{(t-1)}\right] + \sigma^2 / (\gamma_j n_t)} \\
&= \frac{1}{\left(\E\left[e_j^{(t-1)}\right]\right)^{-1} + \left(\sigma^2 /(\gamma_j n_t)\right)^{-1}}.
\end{aligned}
\end{equation}
\end{enumerate}

To prove the final results, we consider mathematical induction.
By Assumption \ref{assump:fullrank}, for each $j\in[p]$, there exists $\tau_j\in[T]$ such that $\tau_j>0$. Therefore, by the above derivation, $e_j^{(\tau_j)}$ satisfies
$$
\begin{aligned}
\E\left[e_j^{(\tau_j)}\right] &= \frac{1}{(e_j^{(0)})^{-1} + (\sigma^2/(\gamma_j^{(\tau_j)}n_{\tau_j}))^{-1}} \\ 
&= \frac{\sigma^2}{\sigma^2/e_j^{(0)} + \gamma_j^{(\tau_j)} n_{\tau_j}} \\
&= \frac{\sigma^2}{\sigma^2/e_j^{(0)} + \gamma_j^{(1)}n_1 + \dots + \gamma_j^{(\tau_j)} n_{\tau_j}}.
\end{aligned}
$$

For $t<\tau_j$, by Case (1) discussed above we have
$$
\E\left[e_j^{(t)}\right] = \frac{\sigma^2}{\sigma^2/e_j^{(0)} + \gamma_j^{(1)}n_1+\dots+\gamma_j^{(t-1)}n_{t}}
$$
since $\gamma_j^{(t)}=0$ for every $t<\tau_j.$

Now suppose
$$
\E\left[e_j^{t}\right] = \frac{\sigma^2}{\sigma^2/e_j^{(0)} + \gamma_j^{(1)}n_1+\dots+\gamma_j^{(t)}n_{t}}
$$
holds for some $t\geq\tau_j$. If $\gamma_j^{(t+1)}=0$, by Case (1) we have 
$$
\E\left[e_j^{(t+1)}\right] = \E e_j^{(t)} = \frac{\sigma^2}{\sigma^2/e_j^{(0)} + \gamma_j^{(1)} n_1 + \dots + \gamma_j^{(t+1)} n_{t+1}}
$$
since $\gamma_j^{(t+1)}=0$. If $\gamma_j^{(t+1)}>0$, by Case (2) we have
$$
\begin{aligned}
\E \left[e_j^{(t+1)}\right] &= \frac{1}{\left(E e_j^{(t)}\right)^{-1} + \left(\sigma^2/\gamma_j^{(t+1)} 
 n_{t+1}\right)^{-1}} \\
&= \frac{\sigma^2}{\sigma^2/e_j^{(0)} + \gamma_j^{(1)} n_1 + \dots + \gamma_j^{(t+1)} n_{t+1}}.
\end{aligned}
$$
Therefore, we conclude that for each $t\in[T]$, 
$$
\E [e_j^{(t)}] = \frac{\sigma^2}{\sigma^2/e_j^{(0)} + \gamma_j^{(1)} n_1 + \dots + \gamma_j^{(t)} n_{t}}
$$
holds if $\Lambda_t$ is chosen by 
$$
\begin{aligned}
\lambda_j^{(t)} &= \frac{\sigma^2/n_t}{E e_j^{(t-1)}} \\
&= \frac{\sigma^2/e_j^{(0)} + \gamma_j^{(1)}n_1+\dots +\gamma_j^{(t-1)}n_{t-1}}{n_t}\\
\end{aligned}.
$$
Finally, since $\|\wh{\vw}_t^{\mathrm{GR}})-\vw_*\|^2 = \sum_{j=1}^p e_j^{(t)}$, taking summation of all $e_j^{(t)}$ gives
$$
\gL(\wh{\vw}_t^{(\mathrm{GR})}) = \sum_{j=1}^p \frac{\sigma^2}{\sigma^2/e_j^{(0)} + \gamma_j^{(1)}n_1 + \dots + \gamma_j^{(t)}n_t}.
$$
\end{proof}

\begin{proof}[Proof of Theorem \ref{thm:approx}]
Recall from the proof of Theorem \ref{thm:adaptive-ridge} that
$$
\begin{aligned}
\E \left[e_j^{(t)}\right] &= \E (u_j^\top (X_t^\top X_t + n H_t)^{-1} X_t \varepsilon_t)^2 + \E (u_j^\top (X_t^\top X_t + n H_t)^{-1} n H_t (\wh{w}_{t-1}-w_*))^2 \\
&= \E (u_j^\top U(n\Gamma_t + n\Lambda_t)^{-1} U^\top X_t \varepsilon_t)^2 + n^2 \E (u_j^\top U(n\Gamma_t + n\Lambda_t)^{-1} \Lambda U^\top (\wh{w}_{t-1}-w_*))^2 \\
&= \frac{\gamma_j \sigma^2 /n + \lambda_j^2 \E e_j^{(t-1)}}{(\lambda_j + \gamma_j)^2} \\
&= \frac{\gamma_j \sigma^2 /n + (\lambda_j + \gamma_j - \gamma_j)^2 \E e_j^{(t-1)}}{(\lambda_j + \gamma_j)^2} \\
&= \E e_j^{(t-1)} - 2 \frac{\gamma_j \E e_j^{(t-1)}}{\lambda_j + \gamma_j} + \frac{\gamma_j^2 \E e_j^{(t-1)} + \gamma_j \sigma^2/n}{(\lambda_j + \gamma_j)^2}.
\end{aligned}
$$
The optimal $\lambda_j$ that minimize the above equation satisfies
$$
\frac{1}{\lambda_j + \gamma_j} = \frac{\E e_j^{(t-1)}}{\gamma_j \E e_j^{(t-1)} + \sigma^2/n_t},
$$
namely
$$
\lambda_j = \frac{\sigma^2/n_t}{\E e_j^{(t-1)}}.
$$
Now suppose we use its approximated version instead:
$$
\frac{1}{\wt{\lambda}_j+\gamma_j} = \frac{1}{{\lambda}_j+\gamma_j} + \Delta
$$
for some $\wt{\lambda}_j$.
Then we have
$$
\begin{aligned}
\E e_j^{(t)} &\leq \frac{\E e_j^{(t-1)} \cdot \sigma^2/(\gamma_j n_t)}{\E e_j^{(t-1)} + \sigma^2/(\gamma_j n_t)} + (\gamma_j^2 \E e_j^{(t-1)} + \gamma_j \sigma^2/n )\Delta^2.
\end{aligned}
$$

Suppose that 
$$
\E e_j^{(t-1)} \leq \frac{C\sigma^2}{e_j^{(0)} + \gamma_j^{(1)} n_1 + \dots \gamma_j^{(t-1)} n_{t-1}}.
$$
If we want
$$
\E e_j^{(t)} \leq \frac{C\sigma^2}{e_j^{(0)} + \gamma_j^{(1)} n_1 + \dots \gamma_j^{(t)} n_t},
$$
holds true, we only need to make sure
$$
\begin{aligned}
&\frac{\E e_j^{t-1} \cdot \sigma^2/(\gamma_j n_t)}{\E e_j^{t-1} + \sigma^2/(\gamma_j n_t)} + (\gamma_j^2 \E e_j^{(t-1)} + \gamma_j \sigma^2/n )\Delta^2 \\
\leq & \frac{C\sigma^2}{e_j^{(0)} + \gamma_j^{(1)} n_1 + \dots + \gamma_j^{(t-1)} n_{t-1} + C \gamma_j^{(t)} n_t} + \gamma_j^2 \Delta^2 \left( \frac{C\sigma^2}{e_j^{0} + \gamma_j^1 n_1 + \dots \gamma_j^{t-1} n_{t-1}} + \frac{\sigma^2}{\gamma_j^t n_t} \right)
\\
\leq & \frac{C\sigma^2}{e_j^{(0)} + \gamma_j^{(1)} n_1 + \dots + \gamma_j^{(t)} n_t}.
\end{aligned}
$$
Define 
$$
\rho_j^{(t)} := \frac{\gamma_j^{(t)} n_t}{e_j^{(0)} + \gamma_j^{(1)} n_1 + \dots + \gamma_j^{(t)} n_t},
$$
then the above inequality becomes
$$
\begin{aligned}
\frac{C}{1+C \rho_j^{(t)}} + (\gamma_j^{(t)})^2 \Delta^2 \left(C + \frac{1}{\rho_j^{(t)}}\right) \leq \frac{C}{1+\rho_j^{(t)}},
\end{aligned}
$$
which is indeed
$$
(\gamma_j^{(t)})^2\Delta^2 \leq \frac{C(C-1)(\rho_j^{(t)})^2}{(1+\rho_j^{(t)})(1+C\rho_j^{(t)})^2}.   
$$
\end{proof}

\section{Proofs for Section \ref{sec:earlystopping}}
\begin{proof}[Proof of Theorem \ref{thm:earlystopping}]
For each $t\in[T]$ and $\tau\in[m_t]$, using the update iteration of MN estimator we have
$$
\begin{aligned}
\vw_t^{(\tau)} - \vw_* &= (\mI_p - \mA_t \mX_t^\top \mX_t/n_t) \vw_t^{(\tau-1)} + (\mA_t/n_t) \mX_t^\top (\mX_t \vw_* + \mvarepsilon_t) - \vw_* \\
&= (\mI_p - \mA_t \mX_t^\top \mX_t/n_t) (\vw_t^{(\tau-1)} - \vw_*) + \mA_t \mX_t^\top \mvarepsilon_t/n_t \\
&= (\mI_p - \mA_t \mX_t^\top \mX_t/n_t)^{\tau}(\vw_{t-1}^{(\mathrm{ES})}-\vw_*) + (\mI - (\mI-\mA_t \mX_t^\top \mX_t/n_t)^\tau)(\mA_t \mX_t^\top \mX_t/n_t)^{-1} \frac{\mA_t}{n_t} \mX_t^\top \mvarepsilon_t \\
&= \mU (\mI_p - \mS_t \boldsymbol{\Gamma}_t)^{\tau} \mU^\top (\vw_{t-1}^{(\mathrm{ES})}-\vw_*) + \mU(\mI_p - (\mI_p-\mS_t \boldsymbol{\Gamma}_t)^\tau )\boldsymbol{\Gamma}_t^{-1} \mU^\top \frac{\mX_t^\top}{n_t} \mvarepsilon_t.
\end{aligned}
$$
Therefore, for each $j=1,\dots,p$, we have
$$
\begin{aligned}
\vu_j^\top (\vw_t^{(\tau)}-\vw_*) = (1-s_j \gamma_j)^\tau \vu_j^\top (\vw_{t-1}^{(\mathrm{ES})}-\vw_*) + (1-(1-s_j \gamma_j)^\tau) \vu_j^\top \frac{\mX_t^\top}{\gamma_j n}\mvarepsilon_j.
\end{aligned}
$$
Note that By the proof of Theorem \ref{thm:adaptive-ridge}, the solution of generalized $\ell_2$ regularization estimator satisfies
$$
\vu_j^\top (\wh{\vw}_t^{(\mathrm{GR})} - \vw_*) = (\gamma_j + \lambda_j)^{-1} \lambda_j \vu_j^\top (\wh{\vw}_{t-1}^{(\mathrm{GR})} - \vw_*) + (\gamma_j + \lambda_j)^{-1} \gamma_j \vu_j^\top \frac{\mX_t^\top}{\gamma_j n} \mvarepsilon_t.
$$
Therefore, if $\lambda_j$ and $s_j$ satisfy
$$
(1-s_j\gamma_j)^{m_t} = \frac{\lambda_j}{\gamma_j + \lambda_j},
$$
namely
$$
s_j = \frac{1-(\lambda_j/(\gamma_j + \lambda_j))^{1/m_t}}{\gamma_j}
$$
or 
$$
\lambda_j = \frac{\gamma_j (1-s_j \gamma_j)^{m_t}}{1-(1-s_j \gamma_j)^{m_t}},
$$
the early stopping and $\ell_2$ regularization output the same estimator, i.e., 
$$
\vw_t^{(\mathrm{ES})} = \vw_t^{(m_t)} = \wh{\vw}_t^{(\mathrm{GR})}.
$$
\end{proof}

\begin{proof}[Proof of Corollary \ref{coro:es}]
By Theorem \ref{thm:adaptive-ridge} and Theorem \ref{thm:earlystopping}, if $s_j^{(t)}$ and $m_t$ satisfy 
$$
\left(1-s_j^{(t)} \gamma_j\right)^{m_t} = \frac{\sigma^2/(\gamma_j n)}{\E\left[e_j^{(t-1)}\right] + \sigma^2/(\gamma_j n)},
$$
the ES estimator $\wh{\vw}_t^{(\mathrm{ES})}$ equals the optimal generalized $\ell_2$ regularization estimator defined in Theorem \ref{thm:adaptive-ridge}.
In this case, its estimation error satisfies
$$
\gL(\wh{\vw}_t^{(\mathrm{ES})}) = \sum_{j=1}^p \frac{\sigma^2}{\sigma^2/e_j^{(0)} + \gamma_j^{(1)}n_1 + \dots + \gamma_j^{(t)}n_t}.
$$
\end{proof}

\section{Proof for Section \ref{sec:extension}}
\begin{proof}[Proof of Theorem \ref{thm:extension}]
For simplicity, we omit the superscript of the GR estimator in this proof.
Let $\mU_t\boldsymbol{\Gamma}_t \mU_t^\top$ be the eigendecomposition of $\boldsymbol{\Sigma}_t$ and define $e_{j,t_1}^{(t_2)}:=((\vu_j^{(t_1)})^\top (\wh{\vw}_{t_2} - \vw_*))^2$ as the projected error of $\vw_{t_2}$ onto the $j$th eigenvector of $\vw_{t_1}$. 
Note that if Assumption \ref{assump:commutable} holds, $\vu_j^{(t_1)}=\vu_j^{(t_2)}$ for every $t_1,t_1\in[T]$ and $e_{j,t_1}^{(t_2)}$ equals to $e_j^{(t_2)}$ defined in Section \ref{sec:main} for each $t_1$.

By the same derivation of \eqref{eq:iteration}, we directly have
$$
\E\left[e_{j,t}^{(t)}\right] = \frac{1}{\left(\E\left[ e_{j,t}^{(t-1)}\right]\right)^{-1} + (\sigma^2/(\gamma_j n_t))^{-1}} \leq \E\left[e_{j,t}^{(t-1)}\right].
$$
Therefore, summing them up with respect to $j$ gives
$$
\gL(\wh{\vw}_t) \leq \gL(\wh{\vw}_{t-1}).
$$
Obviously, the inequality holds strictly as long as $\sum_{j=1}^p \gamma_j^{(t)}>0$, i.e., there exists one $j$ such that $\gamma_j>0$.

We remark that if $e_{j,t}^{(t-1)}=e_{j,t-1}^{(t-1)}$ for every $t$, we can use mathematical induction to derive the estimation error in Theorem \ref{thm:adaptive-ridge}.
However, without Assumption \ref{assump:commutable}, we cannot ensure this equation holds true.

\end{proof}

\end{document}